\documentclass[conference,10pt]{IEEEtran}
\IEEEoverridecommandlockouts
\usepackage{float}
\usepackage{amsmath}
\usepackage{amsfonts}
\usepackage{mathtools}
\usepackage{amssymb}
\usepackage{float}
\usepackage{xcolor}
\usepackage{enumerate}
\usepackage{isomath}
\usepackage{bm}
\usepackage{cite}
\usepackage{caption}
\usepackage{subfig}
\usepackage{lipsum}
\usepackage{graphicx}
\usepackage{algorithm,algcompatible,amsmath}
\algnewcommand\INPUT{\item[\textbf{Input:}]}%
\algnewcommand\OUTPUT{\item[\textbf{Output:}]}%
\restylefloat{table}
\usepackage{tikz} 
\usepackage{lipsum}
\newcommand\blfootnote[1]{%
  \begingroup
  \renewcommand\thefootnote{}\footnote{#1}%
  \addtocounter{footnote}{-1}%
  \endgroup
}
\usetikzlibrary{arrows, positioning, automata}
\usetikzlibrary{shapes.geometric, arrows}
\def\xfoo#1^#2\relax#3\valign{%
\mathbf{#1}\ifx\valign#2\valign\else^{\mathbf{#2}}\fi}

\def\KL{{\rm KL}}
\def\JS{{\rm JS}}
\def\Zm{S} 
\def \Pzt{P_{S|\tau}} 
\def \PzTt{P_{S|T=\tau}} 
\def \PzT{P_{S|T}} 
\def \PZmTi{P_{\Zm_i|\tau}} 
\def \PZMTi{P_{\Zm_i|T}} 
\def \PZmti{P_{\Zm_i|\tau_i}} 
\def \PZMti{P_{\Zm_i|T_i}} 

\def\Xbf{\mathbf{X}}
\def\Ybf{\mathbf{Y}}
\def\mset{{\Zm_{1:N}}}

\def\avg{{\rm avg}}

\def \Nscr{\mathcal{N}}
\def \Wscr{\mathscr{W}}

\usepackage{amsmath, amssymb, bbm, xspace}
\usepackage{epsfig}
\usepackage{longtable}
\usepackage{color}
\usepackage{mathrsfs}
\usepackage{comment}

\usepackage{courier}

\newtheorem{theorem}{Theorem}[section]
\newtheorem{assumption}{Assumption}[section]

\newtheorem{lemma}{Lemma}[section]

\newtheorem{corollary}[theorem]{Corollary}

\newtheorem{definition}{Definition}[section]

\def\bkE{{\rm I\kern-.17em E}}
\def\bk1{{\rm 1\kern-.17em l}}
\def\bkD{{\rm I\kern-.17em D}}
\def\bkR{{\rm I\kern-.17em R}}
\def\bkP{{\rm I\kern-.17em P}}

\def\bkZ{{\bf{Z}}}

\def\bkE{{\rm I\kern-.17em E}}
\def\bk1{{\rm 1\kern-.17em l}}
\def\bkD{{\rm I\kern-.17em D}}
\def\bkR{{\rm I\kern-.17em R}}
\def\bkP{{\rm I\kern-.17em P}}

\makeatletter
\newcommand{\pushright}[1]{\ifmeasuring@#1\else\omit\hfill$\displaystyle#1$\fi\ignorespaces}
\newcommand{\pushleft}[1]{\ifmeasuring@#1\else\omit$\displaystyle#1$\hfill\fi\ignorespaces}
\makeatother

\def\bkZ{{\bf{Z}}}
\def\b12{(\beta_1,\beta_2)}

\newenvironment{proofarg}[1][]{\noindent\hspace{2em}{\itshape Proof #1: }}{\hspace*{\fill}~\qed\par\endtrivlist\unskip}

\newcounter{example}
\renewcommand{\theexample}{\thesection.\arabic{example}}
\newenvironment{examplec}[1][]{\refstepcounter{example}
\par\medskip \noindent%
   \textbf{Example~\theexample. #1} \rmfamily}{\hfill $\square$   \hspace{-4.5pt} \vspace{6pt}}

\newcounter{remark}
\renewcommand{\theremark}{\thesection.\arabic{remark}}

\def\Hscr{\mathscr{H}}

\def\Ebb{\mathbb{E}}
\newlength{\noteWidth}
\long\def\notes#1{\ifinner
{\tiny #1}
\else
\marginpar{\parbox[t]{\noteWidth}{\raggedright\tiny #1}}
\fi\typeout{#1}}

 \def\notes#1{\typeout{read notes: #1}} 

\newcommand{\ie}{i.e.\@\xspace} 
\newcommand{\etal}{et al.\@\xspace} 

\newcommand{\Real}{\ensuremath{\mathbb{R}}}

\def\Ebb{\mathbb{E}}

\def\exp{\mathop{\hbox{\rm exp}}}

\def\spose#1{\hbox to 0pt{#1\hss}}

\def\text #1{\hbox{\quad#1\quad}}

\def\nthinsp{\mskip -2   mu}

\def\superstar{^{\raise 0.5pt\hbox{$\nthinsp *$}}}
\def\SUPERSTAR{^{\raise 0.5pt\hbox{$*$}}}

\def\lamstarT {\lambda^{\raise 0.5pt\hbox{$\nthinsp *$}T}}

\def\Lscr{{\cal L}}

\def\Pscr{{\cal P}}

\def\Uscr{{\cal U}}

\def\Wscr{{\cal W}}

\def\Nscr{{\cal N}}

\def\Zscr{{\cal Z}}

\def\non{\nonumber}

\let\forallnew\forall
\renewcommand{\forall}{\forallnew\ }
\let\forall\forallnew

		\def\bkE{{\rm I\kern-.17em E}}
		\def\bk1{{\rm 1\kern-.17em l}}
		\def\bkD{{\rm I\kern-.17em D}}
		\def\bkR{{\rm I\kern-.17em R}}
		\def\bkP{{\rm I\kern-.17em P}}
		\def\bkY{{\bf \kern-.17em Y}}
		\def\bkZ{{\bf \kern-.17em Z}}
		\def\bkC{{\bf  \kern-.17em C}}

{\begin{list}{}%
         {\setlength{\leftmargin}{#1}}%
         \item[]%
}
{\end{list}}

		\def\bsp{\begin{split}}
		\def\beq{\begin{eqnarray}}
		\def\bal{\begin{align*}}
		\def\bc{\begin{center}}
		\def\be{\begin{enumerate}}
		\def\bi{\begin{itemize}}
		\def\bs{\begin{small}}
		\def\bS{\begin{slide}}
		\def\ec{\end{center}}
		\def\ee{\end{enumerate}}
		\def\ei{\end{itemize}}
		\def\es{\end{small}}
		\def\eS{\end{slide}}
		\def\eeq{\end{eqnarray}}
		\def\eal{\end{align*}}
		\def\esp{\end{split}}
		\def\qed{ \vrule height7.5pt width7.5pt depth0pt}  

	\def\cp2problem#1#2#3#4{\fbox
		 {\begin{tabular*}{0.9\textwidth}
			{@{}l@{\extracolsep{\fill}}l@{\extracolsep{6pt}}l@{\extracolsep{\fill}}c@{}}
				#1 & & $#4 $ 
			\end{tabular*}}}

		\def\bkE{{\rm I\kern-.17em E}}
		\def\bk1{{\rm 1\kern-.17em l}}
		\def\bkD{{\rm I\kern-.17em D}}
		\def\bkR{{\rm I\kern-.17em R}}
		\def\bkP{{\rm I\kern-.17em P}}
		
		\def\bkZ{{\bf{Z}}}

\newcommand {\beeq}[1]{\begin{equation}\label{#1}}
\newcommand {\eeeq}{\end{equation}}
\newcommand {\bea}{\begin{eqnarray}}
\newcommand {\eea}{\end{eqnarray}}

\def\texitem#1{\par\smallskip\noindent\hangindent 25pt
               \hbox to 25pt {\hss #1 ~}\ignorespaces}

\def\bsp{\begin{split}}
		\def\beq{\begin{eqnarray}}
		\def\bal{\begin{align*}}
		\def\bc{\begin{center}}
		\def\be{\begin{enumerate}}
		\def\bi{\begin{itemize}}
		\def\bs{\begin{small}}
		\def\bS{\begin{slide}}
		\def\ec{\end{center}}
		\def\ee{\end{enumerate}}
		\def\ei{\end{itemize}}
		\def\es{\end{small}}
		\def\eS{\end{slide}}
		\def\eeq{\end{eqnarray}}
		\def\eal{\end{align*}}
		\def\esp{\end{split}}
		\def\qed{ \vrule height7.5pt width7.5pt depth0pt}  

\usepackage{amsmath, amssymb, xspace}
\usepackage{epsfig}
\usepackage{longtable}
\usepackage{color}
\usepackage{mathrsfs}
\usepackage{subfig}
\newenvironment{proof}[1][]{{\noindent \textit{ Proof}: }}{\hfill \qed \vspace{3pt}\\ }

\def\Nscr{{\cal N}}

\ifCLASSINFOpdf
  \else
 
\fi

\author{\IEEEauthorblockN{Sharu Theresa Jose and Osvaldo Simeone}}
\title{An Information-Theoretic Analysis of the Impact of Task Similarity on Meta-Learning}
\vspace{-0.1cm}
\begin{document}
\maketitle
\begin{abstract}
Meta-learning aims at optimizing the hyperparameters of a model class or training algorithm from the observation of data from a number of related tasks. Following the setting of Baxter\cite{baxter2000model}, the tasks are assumed to belong to the same \textit{task environment}, which is defined by a distribution over the space of tasks and by per-task data distributions. The statistical properties of the task environment thus dictate the similarity of the tasks. The goal of the meta-learner is to ensure that the hyperparameters obtain a small loss when applied for training of a new task sampled from the task environment.
The difference between the resulting average loss, known as meta-population loss, and the corresponding empirical loss measured on the available data from related tasks, known as \textit{ meta-generalization gap}, is a measure of the generalization capability of the meta-learner. In this paper, we present novel information-theoretic bounds on the \textit{average absolute value of the meta-generalization gap}. Unlike prior work \cite{jose2020information}, our bounds explicitly capture the impact of task relatedness, the number of tasks, and the number of data samples per task on the meta-generalization gap. Task similarity is gauged via the Kullback-Leibler (KL) and Jensen-Shannon (JS) divergences. We illustrate the proposed bounds on the example of ridge regression with meta-learned bias.
\end{abstract}
\blfootnote{The authors are with King's Communications, Learning, and Information Processing (KCLIP) lab at the Department of Engineering of King’s College London, UK (emails: sharu.jose@kcl.ac.uk, osvaldo.simeone@kcl.ac.uk).
The authors have received funding from the European Research Council
(ERC) under the European Union’s Horizon 2020 Research and Innovation
Programme (Grant Agreement No. 725731).}
\vspace{-0.5cm}
\section{Introduction}

Conventional learning optimizes model parameters using a training algorithm, while meta-learning 
 optimizes the hyperparameters of a training algorithm. A meta-learner has access to data from a class of tasks, and its goal is to ensure that the resulting training algorithm perform well on any new tasks from the same class. To elaborate, consider 
  an arbitrary learning algorithm, referred to as base-learner, as a stochastic mapping $P_{W|\Zm,U}$ from the input training data set $\Zm$ to the output model parameter $W$ for a given hyperparameter vector $U$. For example, the base-learner may be a stochastic gradient descent (SGD) algorithm with the hyperparameter vector $U$ defining the initialization \cite{finn2017model} or the learning rate \cite{li2017meta}.
For a fixed base-learner and a fixed meta-learner (to be formallly defined below), we ask: \textit{Given the level of similarity of the tasks in a given class, how many tasks and how much data per task should be observed to guarantee that the target average population loss for new tasks can be well approximated using the available meta-training data?}

\begin{figure}[t]
 \centering 
   \includegraphics[scale=0.3,trim=1.2in 2.7in 0.3in 0.55in,clip=true]{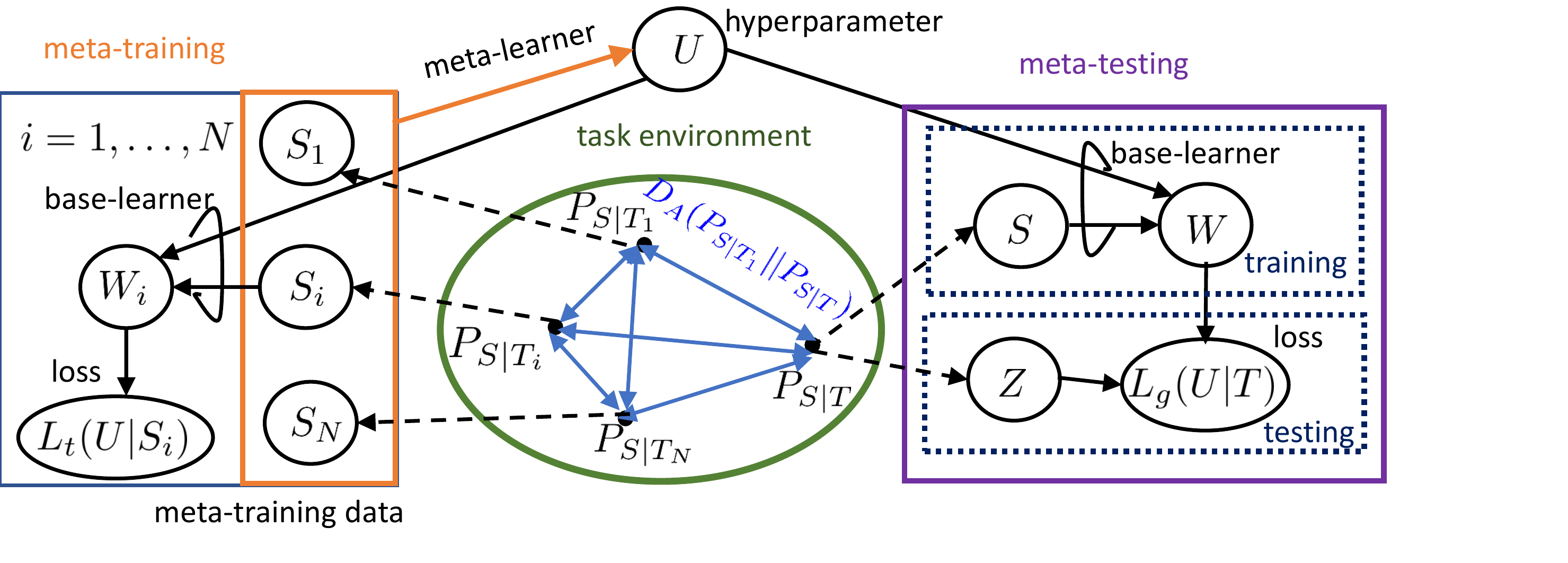} 
   \caption{Overview of the meta-learning problem setup. } \label{fig:overview}
   \vspace{-0.7cm}
  \end{figure}  
  
As illustrated in Figure~\ref{fig:overview}, a meta-learner observes data sets $\mset=(\Zm_1,\hdots,\Zm_N)$ from $N$ tasks $T_{1:N}=(T_1,\hdots,T_N)$, generated according to their respective data distributions. Based on the meta-training set $\mset$, the meta-learner determines a vector of hyperparameter $U$. Accordingly, the meta-learner is defined as a stochastic mapping $P_{U|\mset}$ from the input meta-training set to the output space of hyperaparameters.  The performance of a hyperparameter $U$ is evaluated in terms of 
the \textit{meta-population loss}, $\Lscr_g(U|T)$, which is the average loss of the base-learner $P_{W|S,U}$ when applied on the training data set $\Zm$ of a new, \textit{meta-test}, task $T$. However, the meta-learner does not have access to the data distribution of the new task, but only to the meta-training set $\mset$. Based on this set, the meta-learner can evaluate the empirical \textit{meta-training loss}, $\Lscr_t(U|\mset)$, obtained with hyperparameter $U$. The difference between the meta-population loss and the meta-training loss, known as \textit{meta-generalization gap}, $\Delta \Lscr(U|T,\mset)=\Lscr_g(U|T)-\Lscr_t(U|\mset)$,  measures how well the performance of the meta-learner on the meta-training set reflects the meta-population loss. 

The main goal of this work is to relate the number of tasks, $N$, and data points per task, $m$, to the average meta-generalization gap for any arbitrary meta-learner $P_{U|\mset}$ and base-learner $P_{W|S,U}$.
  Following the setting of Baxter \cite{baxter2000model}, the tasks are assumed to belong to a \textit{task environment}, which defines a task probability distribution $P_T$ on the space of tasks $\mathcal{T}$, where each task $T \in \mathcal{T}$ is associated with a data distribution $P_{Z|T}$. The statistical properties of the task environment thus dictate the similarity of the tasks in the class of interest. 
   Intuitively, if the average ``distance'' between data distributions of any two tasks in the task environment is small, the meta-learner should be able to learn a suitable shared hyperparameter $U$ by observing fewer tasks $N$. In line with this observation, 
the main contribution of the present work is a novel information-theoretic upper bound 
on the average meta-generalization gap
 that explicitly depends on measures of task similarity within the the task environment.

\subsection{Related Work}
While information-theoretic upper bounds on the generalization gap for conventional learning have been extensively studied \cite{russo2016controlling,xu2017information,bu2019tightening,lopez2018generalization ,negrea2019information,wang2019information}, there has been limited work on similar bounds for meta-learning. The recent works \cite{jose2020information} and \cite{rezazadeh2020conditional} extend the individual sample mutual information (ISMI) bound of Bu \etal \cite{bu2019tightening} and the conditional mutual information based bound of Steinke \etal \cite{steinke2020reasoning} to meta-learning, respectively. Another line of work includes probably approximately correct (PAC) bounds  based on algorithmic stability \cite{maurer2005algorithmic}, and PAC-Bayesian bounds \cite{pentina2014pac}, \cite{amit2018meta},\cite{rothfuss2020pacoh}. These bounds hold with high probability over meta-training set and tasks, and they are not directly comparable to \cite{jose2020information} and \cite{rezazadeh2020conditional}.  None of the above works explicitly captures the impact of task relatedness in the meta-learning environment. The similarity among tasks is instead well-accounted for in studies of domain adaptation or transfer learning via various measures of divergence, including $\Hscr$-divergence \cite{ben2007analysis}, integral probability metric \cite{zhang2012generalization}, Kullback-Leibler (KL) divergence \cite{wu2020information},\cite{jose2020transfer} and Jensen-Shannon (JS) divergence \cite{jose2020informationtheoretic}. Unlike meta-learning, in transfer learning, the target task is fixed, and therefore the performance bounds for transfer learning do not apply to meta-learning.
\subsection{Main Contributions}
In this work, based on KL divergence and JS divergence-based measures of similarity  between tasks from task environment, we present  novel information-theoretic upper bounds on the average of the absolute value of the meta-generalization gap. The bounds explicitly capture the relationship between the meta-training data set size, the tasks' similarity, and the meta-generalization gap. The derived bounds are illustrated on two numerical examples.
\vspace{-0.3cm}
\section{Problem Definition}\label{sec:pblmdefinition}
In this section, we give a formal definition of the problem of interest by introducing the operations of the base-learner and of the meta-learner, and by defining the meta-generalization gap and measures of task relatedness.
\vspace{-0.3cm}
\subsection{Base-Learner}
Consider a task $\tau \in \mathcal{T}$ with its associated data distribution $P_{Z|T=\tau}\in \Pscr(\Zscr)$\footnote{We use $\Pscr(\cdot)$ to denote the set of all probability distributions on `$\cdot$'.} in the space of data samples $\Zscr$. A base-learner observes a data set $\Zm=(Z_1, \hdots,Z_m) \sim \PzTt$ of $m$  samples drawn i.i.d from the task-specific data distribution $P_{Z|T=\tau}$. Based solely on $\Zm$, without knowledge of the task $\tau$ and of the data distribution $P_{Z|T=\tau}$, the goal of the base-learner is to infer a model parameter $W \in \Wscr$ such that it generalizes well on a test data point $Z \sim P_{Z|T=\tau}$ drawn independent of $\Zm$. The performance of a model parameter $w \in \Wscr$ on a data sample $z \in \Zscr$ is measured by a loss function $l:\Wscr \times \Zscr \rightarrow \Real_{+}$.

We define the base-learner as a stochastic mapping $P_{W|\Zm,u} \in \Pscr(\Wscr)$ from the input training set $\Zm$ to the output space of model parameters $\Wscr$ for a given hyperparameter $u$. While the true goal of the base-learner is to minimize the \textit{population loss},
\begin{align}
L_g(w|\tau)=\Ebb_{P_{Z|T=\tau}}[l(w,Z)], \label{eq:genloss_base}
\end{align} which is the average loss over a test data $Z \sim P_{Z|T}$, this is not computable since the data distribution $P_{Z|T=\tau}$ is unknown. The base-learner evaluates instead the empirical \textit{training loss} 
\begin{align}
L_t(w|\Zm)= \frac{1}{m} \sum_{j=1}^m l(w,Z_j). 
\end{align} The difference between the population loss and the training loss, $\Delta L(w|\Zm,\tau)=L_g(w|\tau)-L_t(w|\Zm)$, is known as the \textit{generalization gap}, and has been widely studied including in the information-theoretic literature \cite{xu2017information,bu2019tightening},\cite{wu2020information}, \cite{jose2020informationtheoretic}.
\subsection{Meta-Learner}
As seen in Figure~\ref{fig:overview}, the goal of the meta-learner is to infer the hyperparameter $u \in \Uscr$ of the base-learner based on data from a number of related tasks from the task environment. Let $\mathcal{T}$ denote the space of tasks. The task environment is defined by a distribution $P_T$ on the set of tasks $\mathcal{T}$ and by the per-task data distributions $\{P_{Z|T=\tau}\}_{\tau \in \mathcal{T}}$. A meta-learner observes a meta-training data set $\mset=(\Zm_1,\hdots, \Zm_N)$ of $N$ data sets. Each $i$th subset $\Zm_i$ is obtained independently by first selecting a task $T_i \sim P_T$ and then generating the dataset $\Zm_i \sim P_{\Zm|T=T_i}$, where $P_{\Zm|T=T_i}=P^{\otimes m}_{Z|T=T_i}$. The meta-learner uses the meta-training data $\mset$ to infer a hyperparameter $u \in \Uscr$. Accordingly, the meta-learner is defined as a stochastic mapping $P_{U|\mset} \in \Pscr(\Uscr)$ from the input meta-training set to the output space $\Uscr$ of hyperparameters. Note that no knowledge of tasks $T_1, \hdots, T_N$ and of the corresponding data distributions $P_{Z|T=T_1},\hdots P_{Z|T=T_N}$ is available at the meta-learner.

The goal of the meta-learner is to ensure that the base-learner $P_{W|\Zm,u}$ with the inferred hyperparameter $u$ performs well on any new, a priori unknown, \textit{meta-test task} $T \sim P_T$ that is independently drawn from the task set $\mathcal{T}$.  
Accordingly,
for any meta-test task $T$, and given hyperparameter vector $u$, the criterion of interest is the
 \textit{meta-population loss}, \ie the average generalization loss
\begin{align}
\Lscr_g(u|T)=\Ebb_{\PzT}\Ebb_{P_{W|\Zm,u}}[L_g(W|T)],
\end{align}
where the average is computed with respect to the training data $\Zm \sim \PzT$ of the meta-test task $T$ and $L_g(W|T)$ is defined as in \eqref{eq:genloss_base}.

 To summarize, the meta-learner $P_{U|\mset}$ uses the meta-training dataset $S_{1:N}$ to obtain hyperparameter $U$. Then, the resulting base-learner $P_{W|S,U}$ uses 
the data $\Zm$ from the meta-test task $T$ to obtain a model parameter $W$ that is tested on a new test point $Z \sim P_{Z|T}$. The meta-population loss corresponds to the loss incurred during meta-testing. 

 To estimate this quantity, the meta-learner evaluates the empirical \textit{meta-training loss} on the meta-training set $\mset$ from the $N$ meta-training tasks $T_{1:N}=(T_1,\hdots,T_N)$ as
\begin{align}
\Lscr_t(u|\mset)&=\frac{1}{N} \sum_{i=1}^N L_t(u|\Zm_i), \quad \mbox{where}
\\
L_t(u|\Zm_i)&=\Ebb_{P_{W|\Zm_i,u}}[L_t(W|\Zm_i)]\label{eq:per-tasktraining}
\end{align} is the average per-task training loss.
\subsection{Meta-Generalization Gap}


The meta-generalization gap for meta-test task $T$ is then defined as the difference
\begin{align}
\Delta \Lscr(u|\mset,T)=\Lscr_g(u|T)-\Lscr_t(u|\mset).
\end{align}
We are interested in studying the average absolute value of the meta-generalization gap, \ie,
\begin{align}
|\overline{\Delta \Lscr}|^{\avg}=\Ebb_{P_T P_{T_{1:N}}} [| \overline{\Delta \Lscr}(T,T_{1:N})  | ], \label{eq:avgmetagengap}
\end{align}
where 
we have defined as \begin{align}\overline{\Delta \Lscr}(T,T_{1:N})=\Ebb_{P_{\mset|T_{1:N}}P_{U|\mset}}[\Delta \Lscr(U|\mset,T)] \label{eq:metagengap_task}\end{align} the average meta-generalization gap for given meta-training tasks $T_{1:N}$ and meta-test task $T$.

   Prior works \cite{jose2020information}, \cite{rezazadeh2020conditional} have adopted the absolute value of the average of \eqref{eq:metagengap_task} over distributions $P_{T_{1:N}}$ and $P_T$, \ie, $|\overline{\Delta \Lscr}^{\avg}|=|\Ebb_{P_T,P_{T_{1:N}}}[\overline{\Delta \Lscr}(T,T_{1:N})]|$ as the performance metric of interest. The metric $|\overline{\Delta \Lscr}^{\avg}|$ ``mixes up'' the tasks by first averaging over meta-training and meta-testing tasks and then taking the absolute value. In contrast, the proposed metric $|\overline{\Delta \Lscr}|^{\avg}$ keeps the contribution of each selection of meta-training and meta-test tasks separate by averaging the respective absolute values. Accordingly, the asymptotic behaviour of $|\overline{\Delta \Lscr}|^{\avg}$ as $m,N \rightarrow \infty$ differs from that of $|\overline{\Delta \Lscr}^{\avg}|$: While $|\overline{\Delta \Lscr}^{\avg}|$ tends to zero by the law of large numbers, this does not generally hold true for $|\overline{\Delta \Lscr}|^{\avg}$. This reflects the important fact that the meta-training loss cannot provide an asymptotically accurate estimate of meta-test loss, which is evaluated on a priori unknown task. 
\subsection{Measures of Task Relatedness}
Given a divergence measure $D_A(p||q)$ between distributions $p$ and $q$, a natural measure of the difference between two tasks $\tau$ and $\tau' \in \mathcal{T}$ is the divergence $D_A(P_{Z|\tau}||P_{Z|\tau'})$ between the data distributions under the two tasks \cite{lucas2020theoretical}. We will specifically use the following definition of task relatedness of a task environment.

\begin{definition}[$\epsilon$-Related Task Environment] \label{def:KL}
A task environment $(P_T, \{P_{Z|T=\tau}\}_{\tau \in \mathcal{T}})$  is said to be $\epsilon$-related if we have the following inequality
\begin{align}
\Ebb_{P_T\cdot P_{T}}\Bigl[D_{A}(P_{S|T}||P_{S|T'})\Bigr] \leq \epsilon, \label{eq:KL}
\end{align}
where the tasks $T$ and $T'$ are independently drawn from the task distribution $P_T$. As a special case, if we have the inequality
\begin{align}
    D_{A}(P_{S|T=\tau}||P_{S|T=\tau'}) \leq \epsilon, \quad \mbox{for all} \hspace{0.1cm} \tau, \tau' \in \mathcal{T},
\end{align} the tasks belong to an $\epsilon$-related task environment. 
\end{definition}

When the KL or JS divergences are used, we will respectively use the terminology $\epsilon$-KL and $\epsilon$-JS related task environment. We recall that the JS divergence is defined as
\begin{align*}
D_{\JS}(p||q)=\frac{1}{2}\bigl[D_{\KL}(p||0.5(p+q))+D_{\KL}(q||0.5(p+q))\bigr].
\end{align*} Due to the tensorization properties of the KL divergence \cite{polyanskiy2014lecture}, inequality \eqref{eq:KL} can be equivalently formulated as $\Ebb_{P_T\cdot P_{T}}[D_{A}(P_{Z|T}||P_{Z|T'})] \leq \epsilon/m$ for $\epsilon$-KL related tasks, while a similar simplification does not apply to $\epsilon$-JS related tasks. The main potential advantage of the JS divergence over the KL divergence is that
the notion of $\epsilon$-KL related task environment, with $\epsilon < \infty$, applies only if all the per-task data distributions $P_{Z|T=\tau}$, for $\tau \in \mathcal{T}$, share the same support. In contrast, the JS divergence between any two distributions is always bounded as 
$D_{\JS}(p||q) \leq \log (2)$ \cite{polyanskiy2014lecture}, 
making the definition of $\epsilon$-JS related task environment with $\epsilon \in (0, \log(2)]$ always applicable.
\begin{lemma}\label{lem:relation}
An $\epsilon$-KL related task environment $(P_T, \{P_{Z|T=\tau}\}_{\tau \in \mathcal{T}})$ is also $\min\{ \log (2), \epsilon/2\}$-JS related. 
\end{lemma}
\begin{proof}
The proof follows directly from  Lin's upper bound \cite{lin1991divergence} on the JS divergence, \ie, $D_{\JS}(p||q) \leq 0.25( D_{\KL}(p||q)+ D_{\KL}(q||p))$.
\end{proof}
\vspace{-0.6cm}
\begin{examplec}\label{ex:1}
Let the data distribution for task $\tau \in \mathcal{T}$ be normally distributed as $P_{Z|T=\tau}=\Nscr(\tau,\nu^2)$ with mean $\tau$ and variance $\nu^2$. The task distribution $P_T$ defines a distribution over the mean parameter $\tau$. Let $\bar{\mu}$ be the  mean and $\bar{\nu}^2$ be the variance of the task distribution $P_T$. We then have the equality
\begin{align}
    \Ebb_{P_T\cdot P_{T}}\Bigl[D_{\KL}(P_{S|T}||P_{S|T'})\Bigr]
    = \frac{m\bar{\nu}^2}{\nu^2},
\end{align}
and hence the task environment is $\epsilon$-KL related if the inequality $m\bar{\nu}^2/\nu^2 \leq \epsilon$ holds. Note that, as the per-task data variance $\nu^2$ decreases for a given task variance $\bar{\nu}^2$, the task dissimilarity parameter $\epsilon$ grows increasingly large. In contrast, by Lemma~\ref{lem:relation}, the task environment is $\epsilon$-JS related with $\epsilon=\min\{\log 2, m \bar{\nu}^2/2\nu^2\}<\infty$.
\end{examplec}
\section{Main Results}\label{sec:mainresults}
In this section, we derive an upper bound on the average absolute value of the meta-generalization gap $|\overline{\Delta \Lscr}|^{\avg}$. The main goal is obtaining information-theoretic insights into the requirements in terms of meta-training data as function of task similarity for an arbitrary meta-learner $P_{U|\mset}$ and base-learner $P_{W|S,U}$.
As is customary in information-theoretic analysis of generalization gaps, we start by making assumptions on the tail probabilities of the loss functions of interest. 
\begin{assumption}\label{assum:1}
 Fix an arbitrary distribution $R_{\Zm_i|\tau,\tau_{1:N}} \in \Pscr(\Zscr^m)$ for $i=1,\hdots,N$ defined on the space of training data sets $\Zscr^m$, which can depend on meta-test task $\tau$ and meta-training tasks $\tau_{1:N}$. For every choice of meta-test task $\tau$ and meta-training tasks $\tau_{1:N}$ in $\mathcal{T}$, the following two conditions hold:
\begin{itemize}
\item[$(a)$] The loss function $l(w,Z)$ is $\delta_{\tau}^2$-sub-Gaussian when $Z \sim P_{Z|T=\tau}$ for all $w \in \Wscr$;
\item [$(b)$] The average per-task training loss $L_t(u|\Zm_i)$ in \eqref{eq:per-tasktraining} is $\sigma^2$-sub-Gaussian when $\Zm_i \sim R_{\Zm_i|\tau,\tau_{1:N}}$ for all $u \in \Uscr$, and for all $i \in \{1,\hdots,N\}$. 
\end{itemize}
\end{assumption}

We note that Assumption~\ref{assum:1}$(a)$ does not in general imply  Assumption~\ref{assum:1}$(b)$.
 Furthermore, if the loss function is bounded, \ie $ a\leq l(\cdot,\cdot)\leq b$ for some scalars $a \geq 0 $ and $b<\infty$, Assumption~\ref{assum:1} holds with $\sigma^2=\delta_{\tau}^2=(b-a)^2/4$ for all $\tau \in \mathcal{T}$ and $i \in \{1,\hdots,N\}$.
\subsection{Bounds on Average Meta-Generalization Gap $\Delta| \Lscr^{\avg}|$}
 In this section, we present our main results, namely an information-theoretic upper bound on the average absolute meta-generalization gap $|\overline{\Delta \Lscr}|^{\avg}$ in \eqref{eq:avgmetagengap}, that depends explicitly on the relatedness of tasks as per Definition~\ref{def:KL}.
 \begin{theorem}\label{thm:generalbound}
 Under Assumption~\ref{assum:1}, the following upper bound on the average meta-generalization gap holds
 \begin{align}
    &|\overline{\Delta \Lscr}|^{\avg} \hspace{-0.1cm}\leq  \frac{1}{N}\hspace{-0.1cm} \sum_{i=1}^N \biggl(\hspace{-0.1cm}\sqrt{2 \sigma^2 \Ebb_{P_T,P_{T_{1:N}}}[D_{\KL}(\PZMTi||R_{\Zm_i|T,T_{1:N}})]}+\non \\& \sqrt{2 \sigma^2 \bigl(I(U;\Zm_i|T_{1:N})\hspace{-0.1cm}+\hspace{-0.1cm}\Ebb_{P_T,P_{T_{1:N}}}[D_{\KL}(\PZMti||R_{\Zm_i|T,T_{1:N}})] \bigr)} \biggr)\non \\& +B, \label{eq:generalbound}
 \end{align}
 where $P_{S_i|T}$ and $P_{S_i|T_i}$ denote the distributions of the data random variable $S_i$ when generated from task $T$ and $T_i$, respectively and \begin{align}
B=\biggl[\frac{1}{m} \sum_{j=1}^m \Ebb_{T'\sim P_T} \sqrt{2 \delta_{T'}^2 I(W;Z_j|T=T',T_{1:N})}\biggr].\label{eq:B}
\end{align}
 \end{theorem}
\begin{proof}
See Appendix~\ref{app:proof}.
\end{proof} 

 To interpret the upper bound \eqref{eq:generalbound}--\eqref{eq:B}, we start by observing that, in a manner similar to \cite{jose2020information}, the theorem is proved by leveraging the following decomposition of the meta-generalization gap $\Delta \Lscr(u|\mset,\tau)$ for any test task $\tau \in \mathcal{T}:$
\begin{align}
\Delta \Lscr(u|\mset,\tau)&=\Lscr_g(u|\tau)-\Lscr_{g,t}(u|\tau)\non \\&+\Lscr_{g,t}(u|\tau)-\Lscr_t(u|\mset). \label{eq:decomposition}
\end{align}
In \eqref{eq:decomposition}, the term $\Lscr_{g,t}(u|\tau)=\Ebb_{\PzTt}[L_t(u|\Zm)]$, with $L_t(u|\Zm)$ as in \eqref{eq:per-tasktraining}, represents the average per-task training loss of the \textit{meta-test} task $\tau$ as a function of the hyperparameter $u$. Therefore, while the first difference in \eqref{eq:decomposition} captures the \textit{within-task generalization gap} for the meta-test task, which results from observing a finite number $m$ of data samples per task; the second difference accounts for the \textit{environment-level generalization gap} from meta-training to meta-test tasks, resulting from the observation of a finite number of tasks $N$.  The upper bound in Theorem~\ref{thm:generalbound} is obtained by separately bounding the two differences in \eqref{eq:decomposition}.

With the decomposition \eqref{eq:decomposition} in mind, the term $B$ in \eqref{eq:generalbound} captures the within-task generalization gap via the conditional mutual information $I(W;Z_j|T=\tau,T_{1:N})$ between the model parameter $W$ and the $j$th sample $Z_j$ of the training set $\Zm$ corresponding to the meta-test task $\tau \in \mathcal{T}$, when the hyperparameter $U\sim P_{U|\mset}$ is randomly selected by the meta-learner trained on tasks $T_{1:N}$. This term is consistent with the standard information-theoretic analyses of conventional learning in \cite{xu2017information}, \cite{bu2019tightening}, and can be interpreted as a measure of the sensitivity of the base-learner's output $W$ to the individual data samples $Z_j$.

In contrast, the environment-level generalization gap is captured by two terms. The first is the conditional mutual information $I(U;\Zm_i|T_{1:N})$ between the hyperparameter $U$ and the $i$th subset $\Zm_i$ of the meta-training set $\mset$ corresponding to tasks $T_{1:N}$. This term can be analogously interpreted as the sensitivity of the meta-learner's output $U$ to the individual meta-training dataset $\Zm_i$. The other two terms include the KL divergence $D_{\KL}(P_{S_i|T}||R_{\Zm_i|T,T_{1:N}})$ between the data distribution $P_{S_i|T}$ of the meta-test task $T$ and the auxiliary distribution $R_{\Zm_i|T,T_{1:N}}$, as well as the divergence $D_{\KL}(P_{S_i|T_i}||R_{\Zm_i|T,T_{1:N}})$ between the data distribution of the $i$th training task $T_i$ and the auxiliary distribution. As we will see next, these two terms allow us to bound the meta-generalization gap as a function of the similarity among tasks as per Definition~\ref{def:KL}.
\subsection{Bounds on $|\overline{\Delta \Lscr}|^{\avg}$ for $\epsilon$-KL and $\epsilon$-JS Related Task Environments}
The following bound holds for $\epsilon$-KL related tasks.
 \begin{corollary}\label{thm:KL_avggap}
If the task environment $(P_T, \{P_{Z|T=\tau}\}_{\tau \in \mathcal{T}})$  is $\epsilon$-KL related, the following bound on the average meta-generalization gap holds under Assumption~\ref{assum:1}
\begin{align}
& |\overline{\Delta \Lscr}|^{\avg} \leq  \frac{1}{N} \sum_{i=1}^N \sqrt{2 \sigma^2 \bigl(I(U;\Zm_i|T_{1:N})+\epsilon \bigr)} +B. \label{eq:epsilon-KLrelatedness_avggap}
\end{align}
 \end{corollary}
 \begin{proof}
 The bound in \eqref{eq:epsilon-KLrelatedness_avggap} follows from \eqref{eq:generalbound} by choosing $R_{\Zm_i|T,T_{1:N}}=P_{S_i|T}$ and by using the definition of $\epsilon$-KL related task environment in \eqref{eq:KL}.
 \end{proof}
 Similarly, we also have the following bound for $\epsilon$-JS related task environment.
 \begin{corollary}\label{thm:JS_avggap}
If the task environment $(P_T, \{P_{Z|T=\tau}\}_{\tau \in \mathcal{T}})$  is $\epsilon$-JS related for $\epsilon \in (0,\log (2)]$, we have the following bound
\begin{align}
&|\overline{\Delta \Lscr}|^{\avg} \leq \frac{2}{N} \sum_{i=1}^N \sqrt{ \sigma^2 \bigl(I(U;\Zm_i|T_{1:N})+2 \epsilon \bigr)} + B. \label{eq:delta-JSrelatedness_avggap}
\end{align}
 \end{corollary}
 \begin{proof}
 To obtain \eqref{eq:delta-JSrelatedness_avggap} from \eqref{eq:generalbound}, inspired by \cite{aminian2020jensen}, we choose the auxiliary distribution $R_{\Zm_i|T,T_{1:N}}=0.5(\PZMTi+\PZMti)$ for $i \in \{1,\hdots,N\}$. Denote $C=\Ebb_{P_T,P_{T_{1:N}}}[D_{\KL}(\PZMTi||R_{\Zm_i|T,T_{1:N}})]$ and $D=I(U;\Zm_i|T_{1:N})+\Ebb_{P_T,P_{T_{1:N}}}[D_{\KL}(\PZMti||R_{\Zm_i|T,T_{1:N}})]$. From the concavity of $\sqrt{2 \sigma^2y}$, we have 
$
 \sqrt{2 \sigma^2 C}+ \sqrt{2 \sigma^2 D} \leq 2 \sqrt{\sigma^2 (C+D)},
$
which, 
 together with the definition of $\epsilon$-JS related task environment, concludes the proof.
 \end{proof}
\vspace{-0.4cm}

 As revealed for the first time by the analysis in this paper, the bounds \eqref{eq:epsilon-KLrelatedness_avggap} and \eqref{eq:delta-JSrelatedness_avggap} demonstrate that the required number of tasks increases with the task dissimilarity parameter $\epsilon$.  We also note that,
 in the asymptotic regime as $m,N \rightarrow \infty$, the bounds in \eqref{eq:epsilon-KLrelatedness_avggap} and \eqref{eq:delta-JSrelatedness_avggap} are non-vanishing, in compliance with the discussion in Section~\ref{sec:pblmdefinition} on the asymptotic behaviour of the metric $|\overline{\Delta \Lscr}|^{\avg}$ as opposed to $|\overline{\Delta \Lscr}^{\avg}|$ considered in \cite{jose2020information}. 
\section{Examples}\label{sec:example}
This section provides further insights by considering two simple examples. To the best of our knowledge, no prior work has studied the average of the absolute value of the generalization gap $|\overline{\Delta \Lscr}|^{\avg}$. Therefore, no bounds exist that can be directly compared. That said, we will provide comparisons with bounds on $|\overline{\Delta \Lscr}^{\avg}|$ derived in \cite{jose2020information}.
\vspace{-0.1cm}
\subsection{Mean Estimation with Meta-Learned Bias}
 We first consider the example of mean estimation of a Gaussian random variable, for which the information-theoretic bounds in \eqref{eq:epsilon-KLrelatedness_avggap} and \eqref{eq:delta-JSrelatedness_avggap} can be computed in closed form. As in Example~\ref{ex:1}, each task $\tau \in \mathcal{T}$ is identified by a data distribution $P_{Z|T=\tau}=\Nscr(\tau,\nu^2)$ and the task distribution is given as $P_T = \Nscr(\bar{\mu},\bar{\nu}^2)$.
Based on the training data set $\Zm_i=(Z_{i,1},\hdots, Z_{i,m})$ of task $\tau_i$, the base-learner outputs the estimate
 $
W_i=\alpha  \tilde{\Zm_i}+(1-\alpha)u, $ which is a convex combination of sample average $\tilde{\Zm_i}=\sum_{j=1}^m Z_{i,j}/m$ and a bias hyperparameter vector $u$, with $0\leq \alpha \leq 1$.  The performance of a model parameter $w$ is measured on a test data point $Z$ according to the loss function
 $l(w,z)=\min\{(w-z)^2,c^2\}$, for some scalar constant $c>0$. 
The meta-learner chooses the bias vector as
$
U= \frac{1}{N} \sum_{i=1}^N \tilde{\Zm_i},
$ which is the empirical average over the data sets of $N$ meta-training tasks. 

Due to the form of the considered loss function, the true average meta-generalization gap cannot be computed in closed form. In contrast, the bounds in \eqref{eq:epsilon-KLrelatedness_avggap} and \eqref{eq:delta-JSrelatedness_avggap} can be computed as follows. Conditioned on a meta-test task $\tau$ and meta-training tasks $\tau_{1:N}$, we have that  $
    \tilde{\Zm_i} \sim \Nscr \bigl(\tau_i, \nu^2/m\bigr)$  and
    $U \sim \Nscr\bigl(N^{-1} \sum_{i=1}^N \tau_i,\nu^2/mN\bigr)  
$, whereby we have
 $I(U;\Zm_i|T_{1:N})=0.5\log \bigl(N/(N-1)\bigr)$. Similarly, it can be seen that 
the mutual information $I(W;Z_j|T=\tau,T_{1:N})$ equals
$
I(W;Z_j|T=\tau,T_{1:N})=0.5 \log (\frac{\alpha^2+(1-\alpha)^2/N}{\alpha^2(m-1)/m+(1-\alpha)^2/N}).
$ Together with $\delta_{\tau}^2=\sigma^2=c^4/4$, the bound in \eqref{eq:epsilon-KLrelatedness_avggap} evaluates as
\begin{align}
\hspace{-0.2cm}\frac{c^2}{\sqrt{2}}\sqrt{\frac{1}{2}\log \frac{N}{N-1}\hspace{-0.05cm}+\epsilon}\hspace{-0.05cm}+\hspace{-0.1cm} \sqrt{\frac{c^4}{4}\log \biggl(\frac{\alpha^2+\frac{(1-\alpha)^2}{N}}{\frac{\alpha^2(m-1)}{m}+\frac{(1-\alpha)^2}{N}} \biggr)} \label{eq:MIbound_KL}
\end{align}
with $\epsilon=m\bar{\nu}^2/\nu^2$. Using Lemma~\ref{lem:relation}, the bound in \eqref{eq:delta-JSrelatedness_avggap} can be similarly evaluated by taking $\epsilon=\min\{ \log(2),m\bar{\nu}^2/2\nu^2\}$. From these experiments, it can be seen that, as $N,m \rightarrow \infty$, the bound \eqref{eq:MIbound_KL} tends to $c^2 \sqrt{\epsilon} /\sqrt{2}$, which is non-vanishing unless the tasks in the environment are identical. This is unlike the bounds on $|\overline{\Delta \Lscr}^{\avg}|$ obtained in \cite{jose2020information}.
\subsection{ Ridge Regression with Meta-Learned Bias}
\begin{figure}[h!]
 \centering 
   \includegraphics[scale=0.4,trim=2.2in 1.2in 1.55in 1in,clip=true]{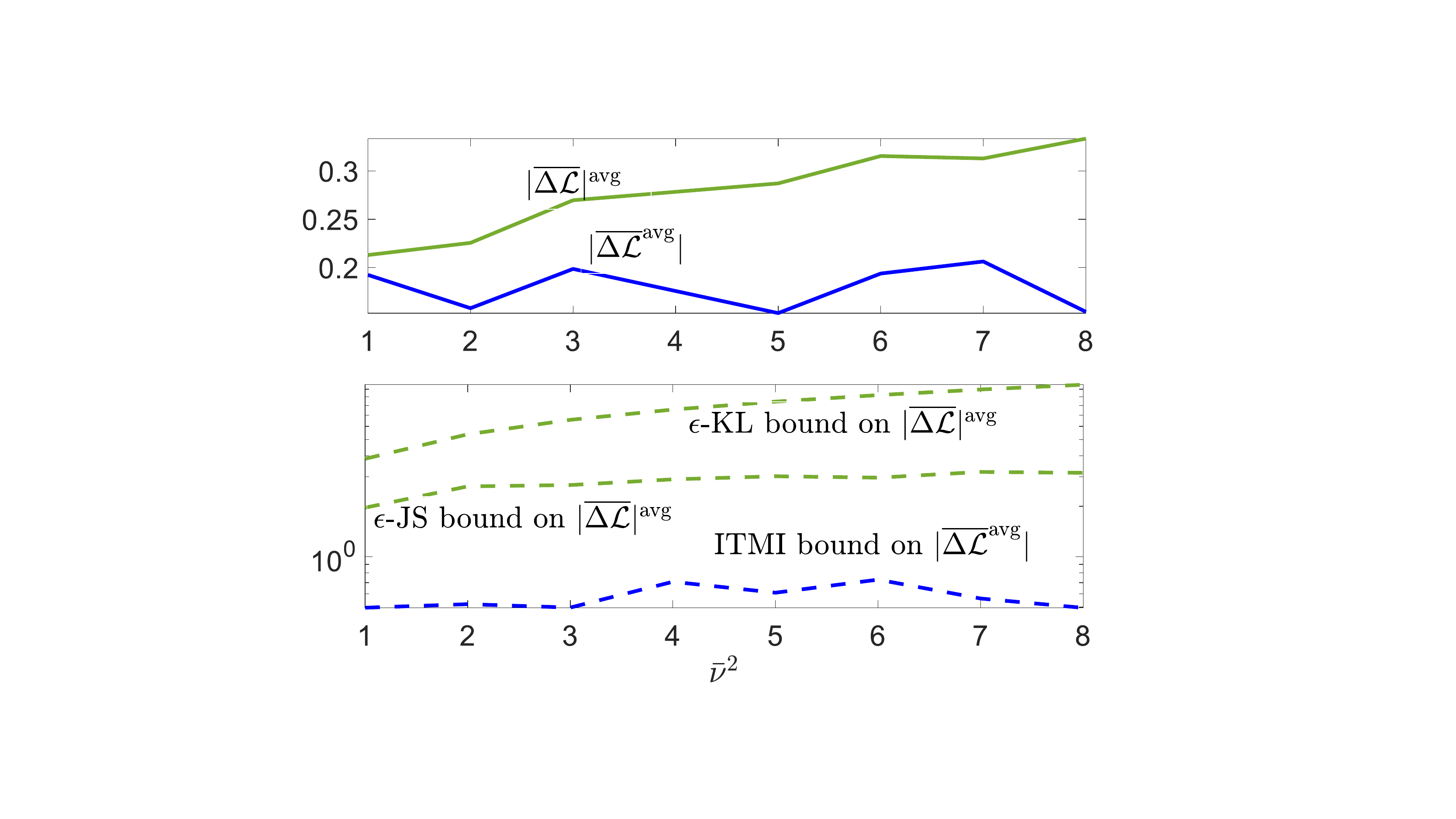} 
   \caption{Comparison of  $|\overline{\Delta \Lscr}|^{\avg}$ and $|\overline{\Delta \Lscr}^{\avg}|$ (top panel) with their corresponding upper bounds (bottom panel) as a function of $\bar{\nu}^2$ ($c=1.5, \lambda=2,\nu^2=1.1$, $N=4,m=6$, $\mu_w=[2 \hspace{0.1cm} 3]$). }
   \label{fig:exp5}
   \vspace{-0.3cm}
  \end{figure} 
We now consider the example of linear regression with the base-learner performing biased regularization \cite{denevi2018incremental,denevi2019learning,denevi2020advantage}. Let each data point $Z=(X,Y)$ be denoted as a tuple of feature vector $X \in \Real^2$ and output label $Y \in \Real$. The base-learner assumes a linear model $f(X)=W'X$, with $W \in \Real^{ 2}$ being the model parameter vector, and $A'$ denoting the transpose of matrix $A$. The model is well-specified, in the sense that, for each task $\tau \in \mathcal{T}$, there exists a true vector $\overline{W} \in \Real^{2}$ such that $X$ is uniformly distributed within the unit circle and $Y|X \sim \Nscr(\overline{W}'X,\nu^2)$.
The task distribution $P_T$ defines a distribution over the true model parameters $\overline{W}$ as $\overline{W} \sim P_T=\Nscr(\mu_w,\bar{\nu}^2 I_2)$ with mean vector $\mu_w \in \Real^2$, covariance parameter $\bar{\nu}^2$ and $I_d$ denoting a $d \times d$ identity matrix. It can be verified that the task environment is $\epsilon$-KL related with $\epsilon=m\bar{\nu}^2/\nu^2$, while the upper bound $\epsilon$ for  JS-related task environment can be estimated from data samples.
The loss accrued by a model parameter $w \in \Real^2$ on a  test data $z \in \Zscr$
is measured using the loss function $l(w,z)=\min\{(w'x-y)^2,c^2\}$
where $c>0$ is some positive constant. Note that the loss function is bounded, and have $\sigma^2=\delta_{\tau}^2=c^4/4$.
 
For any per-task data set $S=(Z_{1},\hdots,Z_{m})$ with $(Z_j=(X_j,Y_j)$, the base-learner selects the minimizer of the following ridge regression problem
\vspace{-0.2cm}
\begin{align}
W^{*}=\arg \min_{w \in \Real^{2}} \frac{1}{m}\sum_{j=1}^m (w'X_{j}-Y_{j})^2+\frac{\lambda}{2} ||w-u||^2_2, \label{eq:baselearner_optimization}
\vspace{-0.1cm}
\end{align}
where $\lambda>0$ is a fixed regularization parameter and $u$ is a hyperparameter bias vector. The optimal value can be computed in closed form as $W^{*}=(2\Xbf'\Xbf/m+\lambda I_2)^{-1}(2\Xbf' \Ybf/m+\lambda u)$ where $\Xbf=[X_{1} \hdots X_{m}]'$ and $\Ybf=[Y_{1} \hdots Y_{m}]'$.
 Note that for the sake of tractability of the optimization problem in \eqref{eq:baselearner_optimization},  the base-learner adopts the squared empirical loss $\sum_{j=1}^m (w'X_{j}-Y_{j})^2/m$ instead of $L_t(w|S)=1/m \sum_{j=1}^m l(w,Z_j)$ as considered in \cite{denevi2020advantage}. This is possible since the bounds in \eqref{eq:epsilon-KLrelatedness_avggap} and \eqref{eq:delta-JSrelatedness_avggap} hold for any arbitrary base-learners and meta-learners.
 
Denoting as $W^{*}_i$ the minimizer in \eqref{eq:baselearner_optimization} for meta-training task $T_i$, the meta-learner selects the bias vector $u$ as the minimizer
 \begin{align}
u^{*}=\arg \min_{u \in \Real^2} \frac{1}{N} \sum_{i=1}^N \frac{1}{m}\sum_{j=1}^m ((W^{* }_i)' X_{i,j}-Y_{i,j})^2 \label{eq:metalearner_optimization}
 \end{align} where $(X_{i,j},Y_{i,j})$ denote the $j$th data sample of the $i$th task.

Figure~\ref{fig:exp5} compares the average absolute meta-generalization gap $|\overline{\Delta \Lscr}|^{\avg}$ with  absolute average meta-generalizaton gap  $|\overline{\Delta \Lscr}^{\avg}|$ studied in \cite{jose2020information}, and the upper bounds in \eqref{eq:epsilon-KLrelatedness_avggap} and \eqref{eq:delta-JSrelatedness_avggap} with the ITMI-based upper bound on $|\overline{\Delta \Lscr}^{\avg}|$ in \cite{jose2020information}, as a function of the covariance parameter $\bar{\nu}^2$ of the task environment. All the quantities are numerically evaluated. It is observed that performance metrics $|\overline{\Delta \Lscr}|^{\avg}$ and $|\overline{\Delta \Lscr}^{\avg}|$ have a distinctly different behavior as the task environment variance $\bar{\nu}^2$, and thus task dissimilarity, increases. In particular, the average meta-generalization loss $|\overline{\Delta \Lscr}^{\avg}|$ appears to be largely insensitive to task dissimilarity, as also predicted by the ITMI bound. In contrast, the metric $|\overline{\Delta \Lscr}|^{\avg}$ studied here reveals the role of task similarity, as captured by the bounds derived in this paper.

\bibliographystyle{IEEEtran}
\bibliography{ref}
\appendices
\section{Proof of Theorem~\ref{thm:generalbound}} \label{app:proof}
To obtain upper bounds on $|\overline{\Delta \Lscr}|^{\avg}=\Ebb_{P_T}\Ebb_{P_{T_{1:N}}}|\overline{\Delta \Lscr}(T,T_{1:N})|$, we bound $|\overline{\Delta \Lscr}(T,T_{1:N})|$ conditioned on test task $T=\tau$ and training tasks $T_{1:N}=\tau_{1:N}$, and then apply Jensen's inequality. Throughout this Appendix, we use $P_{\cdot|\tau,\tau_{1:N}}$ to denote the distribution $P_{\cdot|T=\tau,T_{1:N}=\tau_{1:N}}$ for notational convenience.  We have the following lemma.
\begin{lemma}\label{lem:KLrelatedness}
Under Assumption~\ref{assum:1}, the following bound holds for meta-test task $\tau$ and meta-training tasks $\tau_{1:N}$
\begin{align}
    &|\overline{\Delta \Lscr}(\tau,\tau_{1:N})| \leq  \frac{1}{N} \sum_{i=1}^N \biggl(\sqrt{2 \sigma^2 D_{\KL}(\PZmTi||R_{\Zm_i|\tau,\tau_{1:N}})}+\non \\& \sqrt{2 \sigma^2 \bigl(I(U;\Zm_i|\tau_{1:N})+D_{\KL}(\PZmti||R_{\Zm_i|\tau,\tau_{1:N}}) \bigr)} \biggr)\non \\& +
\frac{1}{m} \sum_{i=1}^m \sqrt{2 \delta_{\tau}^2 I(W;Z_j|\tau,\tau_{1:N})}. \label{eq:conditionedbound}
\end{align}
\end{lemma}
In order to prove Lemma~\ref{lem:KLrelatedness}, we need the following lemma, whose proof is reported later in this appendix. For ensuring that all the  KL divergences and mutual informations defined in \eqref{eq:conditionedbound} are finite and well-defined, and there are no measurability issues while applying change of measure (to be introduced later), we also assume that $P_{S_i|\tau}$ and $R_{S_i|\tau,\tau_{1:N}}$ share the same support for all $\tau, \tau_{1:N} \in \mathcal{T}$. Similarly, we also assume that the joint distribution $P_{U,S_i|\tau_{1:N}}$ and the product of the marginals $P_{U|\tau_{1:N}} P_{S_i|\tau_{i}}$ have the same support for all tasks $\tau_{1:N} \in \mathcal{T}$. Similar assumption also holds for $P_{W,Z|\tau,\tau_{1:N}}$ and $P_{W|\tau,\tau_{1:N}}P_{Z|\tau}$ for all tasks $\tau,\tau_{1:N} \in \mathcal{T}$.
 \begin{lemma}\label{lem:expinequality_JS}
Under Assumption~\ref{assum:1}$(a)$, we have the following exponential inequalities on the per-task training loss $L_t(U|\Zm_i)$ which hold for all $\lambda \in \Real$. For each $i \in \{1,\hdots, N\}$, we have that
\begin{align}
&\Ebb_{P_{U|\tau_{1:N}}\PZmTi}\biggl[\exp \biggl(\lambda (L_t(U|\Zm_i)-\Ebb_{R_{\Zm_i|\tau,\tau_{1:N}}}[L_t(U|\Zm_i)])\non \\& \qquad -\log \frac{\PZmTi}{R_{\Zm_i|\tau,\tau_{1:N}}} -\frac{\lambda^2\sigma^2}{2} \biggr) \biggr]\leq 1 \label{eq:subgaussian_envt_3_JS},
\end{align}
and 
\begin{align}
   & \Ebb_{P_{U,\Zm_i|\tau_{1:N}}}\biggl[\exp \biggl(\lambda (L_t(U|\Zm_i)-\Ebb_{R_{\Zm_i|\tau,\tau_{1:N}}}[L_t(U|\Zm_i)])\non \\& -\log \frac{ \PZmti}{R_{\Zm_i|\tau,\tau_{1:N}}} -\imath(U,\Zm_i|\tau_{1:N})-\frac{\lambda^2\sigma^2}{2} \biggr) \biggr]\leq 1 \label{eq:subgaussian_envt_33_JS},
\end{align}
where $$\imath(U,\Zm_i|\tau_{1:N})=\log \frac{P_{U,\Zm_i|\tau_{1:N}}}{P_{U|\tau_{1:N}}\PZmti}$$ is the information density.
Similarly, we also have that
\begin{align}
&\Ebb_{P_{W,Z|\tau,\tau_{1:N}}}\biggl[\exp \biggl(\lambda(l(W,Z)-\Ebb_{P_{Z|\tau}}[l(W,Z)] )\biggr)\non \\&-\frac{\lambda^2 \delta_{\tau}^2}{2} -\imath(W,Z|\tau,\tau_{1:N}) \biggr]  \leq 1, \label{eq:expinequality_task}
\end{align}
where $$\imath(W,Z|\tau,\tau_{1:N})=\log \frac{P_{W,Z|\tau,\tau_{1:N}}}{P_{W|\tau,\tau_{1:N}}P_{Z|\tau}}$$ is the information density.
\end{lemma}
\begin{proofarg}[of Lemma~\ref{lem:KLrelatedness}]
As discussed in Section~\ref{sec:mainresults}, the idea is to separately bound the two differences in \eqref{eq:decomposition}. We first bound the second difference that captures the environment-level generalization gap. The average of this difference can be equivalently written as
\begin{align}
&\Ebb_{P_{\mset|\tau_{1:N}}P_{U|\mset}}[\Lscr_{g,t}(U|\tau)-\Lscr_t(U|\mset)] \non \\
&=\frac{1}{N} \sum_{i=1}^N \biggl[ \Ebb_{P_{U|\tau_{1:N}}} \Ebb_{\PZmTi}[L_t(U|\Zm_i)]\non \\&-\Ebb_{P_{U,\Zm_i|\tau_{1:N}}}[L_t(U|\Zm_i)]\biggr], \label{eq:1}
\end{align}
where note that the first average is with respect to $P_{U|\tau_{1:N}}$, the marginal of the joint distribution $P_{U,\mset|\tau_{1:N}}=P_{\mset|\tau_{1:N}}P_{U|\mset}$, and $\PZmTi$, the data distribution with respect to the test task $\tau$. On the other hand, the second average in \eqref{eq:1} is with respect to the joint distribution $P_{U,\Zm_i|\tau_{1:N}}$ of the hyperparameter $U$ and $i$th training data $\Zm_i$, which is obtained by marginalizing $P_{U,\mset|\tau_{1:N}}$.

To bound the difference $\Ebb_{P_{U|\tau_{1:N}}} \Ebb_{\PZmTi}[L_t(U|\Zm_i)]-\Ebb_{P_{U,\Zm_i|\tau_{1:N}}}[L_t(U|\Zm_i)]$, we resort to the exponential inequalities in Lemma~\ref{lem:expinequality_JS}. Applying Jensen's inequality on \eqref{eq:subgaussian_envt_3_JS} with $\lambda=\lambda_1>0$ and choosing $\lambda_1=\sqrt{2 D_{\KL}(\PZmTi||R_{\Zm_i|\tau,\tau_{1:N})})}/\sigma$ then yields the following inequality
\begin{align}
    &\Ebb_{P_{U|\tau_{1:N}}\PZmTi}[L_t(U|\Zm_i)]-\Ebb_{P_{U|\tau_{1:N}}}\Ebb_{R_{\Zm_i|\tau,\tau_{1:N}}}[L_t(U|\Zm_i)] \non \\ & \leq \sqrt{2 \sigma^2  D_{\KL}(\PZmTi||R_{\Zm_i|\tau,\tau_{1:N}})} \label{eq:subgaussian_envt_4_JS} .
\end{align}
Similarly, applying Jensen's inequality on \eqref{eq:subgaussian_envt_33_JS} with $\lambda=-\lambda_2$, $\lambda_2>0$ and choosing $\lambda_2=\sqrt{2 (D_{\KL}(\PZmti||R_{\Zm_i|\tau,\tau_{1:N})})+I(U;\Zm_i|\tau_{1:N})}/\sigma$  yields the following inequality
\begin{align}
&\Ebb_{P_{U|\tau_{1:N}}}\Ebb_{R_{\Zm_i|\tau,\tau_{1:N}}}[L_t(U|\Zm_i)]-\Ebb_{P_{U,\Zm_i|\tau_{1:N}}}[L_t(U|\Zm_i)] \non \\& \leq \sqrt{2 \sigma^2  \Bigl(D_{\KL}(\PZmti||R_{\Zm_i|\tau,\tau_{1:N}})+I(U;\Zm_i|\tau_{1:N})\Bigr)} \label{eq:subgaussian_envt_44_JS}.
\end{align}
Adding \eqref{eq:subgaussian_envt_4_JS} and \eqref{eq:subgaussian_envt_44_JS} then yields that
\begin{align}
&\Ebb_{P_{U|\tau_{1:N}}\PZmTi}[L_t(U|\Zm_i)]-\Ebb_{P_{U,\Zm_i|\tau_{1:N}}}[L_t(U|\Zm_i)] \non \\
& \leq \sqrt{2 \sigma^2  D_{\KL}(\PZmTi||R_{\Zm_i|\tau,\tau_{1:N}})}+\non \\&\sqrt{2 \sigma^2  \Bigl(D_{\KL}(\PZmti||R_{\Zm_i|\tau,\tau_{1:N}})+I(U;\Zm_i|\tau_{1:N})\Bigr)} .\label{eq:3}
\end{align}
Similarly, choosing $\lambda=\lambda_3>0$ in \eqref{eq:subgaussian_envt_33_JS} and $\lambda=-\lambda_4$, $\lambda_4>0$, in \eqref{eq:subgaussian_envt_3_JS}, applying Jensen's inequality, optimizing over $\lambda_3,\lambda_4$, and finally adding the resultant inequalities results in the same upper bound on $\Ebb_{P_{U,\Zm_i|\tau_{1:N}}}[L_t(U|\Zm_i)]-\Ebb_{P_{U|\tau_{1:N}}\PZmTi}[L_t(U|\Zm_i)]$  as in \eqref{eq:3}. Together, we thus have that
\begin{align}
&\Bigl|\Ebb_{P_{U|\tau_{1:N}}\PZmTi}[L_t(U|\Zm_i)]-\Ebb_{P_{U,\Zm_i|\tau_{1:N}}}[L_t(U|\Zm_i)] \Bigr|\non \\
& \leq \sqrt{2 \sigma^2  D_{\KL}(\PZmTi||R_{\Zm_i|\tau,\tau_{1:N}})}+\non \\&\sqrt{2 \sigma^2  \Bigl(D_{\KL}(\PZmti||R_{\Zm_i|\tau,\tau_{1:N}})+I(U;\Zm_i|\tau_{1:N})\Bigr)} .\label{eq:4}
\end{align}
Substituting this in \eqref{eq:1} yields an upper bound on the environment-level generalization gap.


We now bound the average within-task generalization gap, \ie, the first difference in \eqref{eq:decomposition}. The average of this difference can be equivalently written as
\begin{align}
&\Ebb_{P_{\mset|\tau_{1:N}}P_{U|\mset}}[\Lscr_{g}(U|\tau)-\Lscr_{g,t}(U|\tau)] \non \\
&=\Ebb_{\Pzt}\Ebb_{P_{U|\tau_{1:N}}P_{W|\Zm,U}}[L_g(W|\tau)-L_t(W|\Zm)]\non \\
&=\Ebb_{\Pzt}\Ebb_{P_{W|\Zm,\tau_{1:N}}}[L_g(W|\tau)-L_t(W|\Zm)]\non \\
&=\frac{1}{m} \sum_{j=1}^m \biggl[\Ebb_{P_{W|\tau,\tau_{1:N} }}\Ebb_{P_{Z_j|\tau}}[l(W,Z_j)]\non \\&-\Ebb_{P_{W,Z_j|\tau,\tau_{1:N}}}[l(W,Z_j)]\biggr], \label{eq:2}
\end{align}
where $P_{W|\Zm,\tau_{1:N}}$ is obtained by marginalizing over $U$ of the joint distribution $P_{W|\Zm,U}P_{U|\tau_{1:N}}$ and $P_{W|\tau,\tau_{1:N}}$ denotes the marginal of $P_{W|\Zm,\tau_{1:N}}\Pzt$. To bound the difference $\Ebb_{P_{W|\tau,\tau_{1:N} }}\Ebb_{P_{Z_j|\tau}}[l(W,Z)]-\Ebb_{P_{W,Z_j|\tau,\tau_{1:N}}}[l(W,Z)]$, we resort to the exponential inequality \eqref{eq:expinequality_task}. Applying Jensen's inequality and optimizing over $\lambda$ with $\lambda= \sqrt{2I(W;Z_j|\tau,\tau_{1:N})}/ \delta_{\tau}$ yields the following bound
\begin{align}
&\Bigl | \Ebb_{P_{W|\tau,\tau_{1:N} }}\Ebb_{P_{Z_j|\tau}}[l(W,Z)]-\Ebb_{P_{W,Z_j|\tau,\tau_{1:N}}}[l(W,Z)] \Bigr |\non \\
& \leq \sqrt{2 \delta_{\tau}^2 I(W;Z_j|\tau,\tau_{1:N})}.
\end{align}
Substituting this in \eqref{eq:2} then yields the second term of \eqref{eq:conditionedbound}.
\end{proofarg}

\begin{proofarg}[of Lemma~\ref{lem:expinequality_JS}]
To obtain the exponential inequalities in \eqref{eq:subgaussian_envt_3_JS} and \eqref{eq:subgaussian_envt_33_JS}, we resort to Assumption~\ref{assum:1}$(a)$. This results in the following inequality for $i=1,\hdots, N$
\begin{align}
&\Ebb_{R_{\Zm_i|\tau,\tau_{1:N}}}\biggl[\exp \biggl(\lambda (L_t(u|\Zm_i)-\Ebb_{R_{\Zm_i|\tau,\tau_{1:N}}}[L_t(u|\Zm_i)]) \non \\& -\frac{\lambda^2\sigma^2}{2} \biggr) \biggr]\leq 1 \label{eq:subgaussian_envt_1_JS},
\end{align} which holds for all $u \in \Uscr$ and $\lambda \in \Real$. 
 We now perform a change of measure from $R_{\Zm_i|\tau,\tau_{1:N}}$ to $P_{S_i|\tau}$ by using the approach adopted in \cite{hellstrom2020generalization}, \cite[Prop.~17]{polyanskiy2014lecture}. 
This results in the inequality
\begin{align}
    &\Ebb_{\PZmTi}\biggl[\exp \biggl(\lambda (L_t(u|\Zm_i)-\Ebb_{R_{\Zm_i|\tau,\tau_{1:N}}}[L_t(u|\Zm_i)])\non \\& \qquad -\log \frac{\PZmTi}{R_{\Zm_i|\tau,\tau_{1:N}}} -\frac{\lambda^2\sigma^2}{2} \biggr) \biggr]\leq 1 \label{eq:subgaussian_envt_2_JS},
\end{align}where recall that $P_{S_i|\tau}$ and $R_{S_i|\tau,\tau_{1:N}}$ are well-defined probability density function (pdf) or probability mass function (pmf).
Similarly, performing a change of measure from $R_{\Zm_i|\tau,\tau_{1:N}}$ to $\PZmti$ results in the inequality
\begin{align}
    &\Ebb_{\PZmti}\biggl[\exp \biggl(\lambda (L_t(u|\Zm_i)-\Ebb_{R_{\Zm_i|\tau,\tau_{1:N}}}[L_t(u|\Zm_i)])\non \\& \qquad -\log \frac{\PZmti}{R_{\Zm_i|\tau,\tau_{1:N}}} -\frac{\lambda^2\sigma^2}{2} \biggr) \biggr]\leq 1 \label{eq:subgaussian_envt_22_JS}.
\end{align}
Averaging \eqref{eq:subgaussian_envt_2_JS} over $U \sim P_{U|\tau_{1:N}}$ yields
\eqref{eq:subgaussian_envt_3_JS}.
Similarly, averaging \eqref{eq:subgaussian_envt_22_JS} over $U \sim P_{U|\tau_{1:N}}$ and performing a change of measure from $P_{U|\tau_{1:N}}\PZmti$ to $P_{U,\Zm_i|\tau_{1:N}}$ yields the inequality \eqref{eq:subgaussian_envt_33_JS}.

To obtain task-level exponential inequality, we have from Assumption~\ref{assum:1} the inequality
\begin{align}
\Ebb_{P_{Z|\tau}}\biggl[\exp \biggl(\lambda(l(w,Z)-\Ebb_{P_{Z|\tau}}[l(w,Z)] )\biggr)-\frac{\lambda^2 \delta_{\tau}^2}{2} \biggr] & \leq 1,
\end{align}
which holds for all $\lambda \in \Real$ and $w \in \Wscr$. Averaging both sides of the inequality over $W \sim P_{W|\tau,\tau_{1:N}}$ and subsequently applying a change of measure from $P_{Z|\tau}P_{W|\tau,\tau_{1:N}}$ to $P_{W,Z|\tau,\tau_{1:N}}$ yields the exponential inequality in \eqref{eq:expinequality_task}.
\end{proofarg}

\end{document}